\documentclass{fundam}

\usepackage{graphicx}
\usepackage{graphicx,xcolor}
 \usepackage{subfigure}
\usepackage{tikz}
\usetikzlibrary{positioning,shapes,chains,backgrounds,arrows,chains,fit,snakes,patterns,shapes.geometric,shapes.symbols}
\usepackage{pgfplots}
\newcommand{\remove}[1]{}
\usepackage[T1]{fontenc}
\usepackage{amsmath}
\usepackage{amsfonts}
\usepackage{amssymb}
\usepackage{flexisym}
\usepackage{mathtools}
\usepackage{comment}
\usepackage{courier}

\usepackage{lscape}
\usepackage{longtable}

\newcommand{\Nat}{\mathbb{I\!\!N}}

\newcommand{\bond}{\!-\!}

\newcommand{\connected}{\mathsf{con}}
\newcommand{\es}{\emptyset}
\newcommand{\state}[2]{\langle {#1}, {#2}\rangle}

\newcommand{\trans}[1]{\ensuremath{\stackrel{#1}{\longrightarrow}}}
\newcommand{\guard}[1]{\mathsf{pre}(#1)}
\newcommand{\effects}[1]{\mathsf{post}(#1)}

\newcommand{\effect}[1]{\mathsf{effect}(#1)}

\newcommand{\RPN}{reversing Petri net }
\newcommand{\btrans}[1]{\ensuremath{\stackrel{#1}{\rightsquigarrow}_{b}}}
\newcommand{\ctrans}[1]{\ensuremath{\stackrel{#1}{\rightsquigarrow}_{c}}}

\newcommand{\proofend}{\hspace*{\fill} $\Box$}

  \usepackage{hyperref}

\begin{document}

\setcounter{page}{273}
\publyear{2021}
\papernumber{2099}
\volume{184}
\issue{4}

  \finalVersionForARXIV

\title{Acyclic and Cyclic Reversing Computations\\in Petri Nets}

\author{Kamila Barylska, Anna Gogoli{\'n}ska\thanks{Address for correspondence:  Faculty of Mathematics
                 and Computer Science, Nicolaus Copernicus University, Toru\'n, Poland. \newline \newline
          \vspace*{-6mm}{\scriptsize{Received August 2021; \ accepted February 2022.}}}
\\
Faculty of Mathematics and Computer Science \\
Nicolaus Copernicus University\\
 Toru\'n, Poland \\
\{kamila.barylska,  anna.gogolinska\}@mat.umk.pl
}

\maketitle

\runninghead{K. Barylska and A. Gogoli{\'n}ska}{Acyclic and Cyclic Reversing Computations in Petri Nets}

\begin{abstract}
Reversible computations constitute an unconventional form of computing where any sequence of
performed operations can be undone by executing in reverse order at any point during a~computation.
It has  been attracting increasing attention as it provides opportunities for
low-power computation, being at the same time essential or
eligible in various applications.
In recent work, we have proposed a structural way of translating Reversing Petri
Nets (RPNs) --  a~type of Petri nets that embeds reversible computation, to bounded Coloured Petri Nets (CPNs)  -- an~extension of traditional Petri Nets, where tokens carry data values.
Three reversing semantics are possible in RPNs: backtracking (reversing of the
lately executed action), causal reversing (action can be reversed only when all its effects have been undone) and out of causal reversing (any previously performed action can be reversed).
In this paper,
we extend the RPN to CPN translation with formal proofs of correctness.
Moreover, the possibility of introduction of cycles to RPNs is discussed.
We analyze which type of cycles could be allowed in RPNs to ensure consistency with the current semantics.
It emerged that the most interesting case related to cycles in RPNs occurs in causal semantics, where various interpretations of
dependency result in different net's behaviour during reversing. Three
definitions of dependence are presented and discussed.
\end{abstract}

\label{sec.int}
\section{Introduction}
\label{sec:RPN}
The classical concept of reversibility in Petri nets is most commonly considered as the ability of a~system to achieve its initial state at any time of any computation (i.e., the initial state is a "home state" \cite{EN}). This property is sometimes also called cyclicity~\cite{A}.
The decades-long research in this area was globally oriented, i.e., it concerned the entire system, not its individual actions \cite{AK,BDE}. On the other hand, in many fields of science, the concept of reversibility is defined for individual system's transitions as the ability to reverse an action, undo its execution, or perform an action "backward" with exactly the same ease as "forward".

Reversible computations are essential in many fields, for example in large parallel simulations~\cite{J} or databases transactions, where withdrawals of some operations are frequently required, like in loss of internet connection during online payments. Reversible computations are also vital part of version control systems, which are widely used in software developing and other disciplines. The whole idea of version control systems is shifting between former and latter versions, hence adding and reversing commits. Other field which attract much interest in reversing computations is
biology. Catalytic reactions can be seen as reversible processes, where binding between the catalyst and the first substrate is reversed after the reaction. Other biological phenomena, where reversing is observed, are for example activation cycle of G-proteins or oxygen transfer by hemoglobin tetramer.

In recent years, substantial work has been underway to study the concept of reversibility in Petri nets in a local context, i.e., focusing attention on a single action and the possibility of undoing it, not on the entire system. The first attempts were to inverse a single system action by adding a strict reverse to it (the same transition, but in opposite direction). The problem of checking whether the set of such obtained reachable markings changes is proven to be undecidable (for unbounded nets), while for coverable markings - decidable. Additionally, it was shown that the set of markings reachable by the system may change after the introduction of just one single strict reverse~\cite{B}. The attention was therefore directed not only on the strict reverses, but also on actions that have exactly the same effect as the reverse (i.e., having isomorphic behaviour - in the meaning of reachability graph)~\cite{C,D,F}. Another area of research involved action reversal in step semantics with auto-concurrency~\cite{E}. Research attention was also given to Petri nets with causal-consistent local reversibility, obtained by unfolding any place-transition net into occurrence nets and folding them back to a coloured Petri net with an infinite number of colours~\cite{G}.
Apart from adding the functionality of reversing (by creating strict or behavioural reverses) to systems modelled with Petri nets, a new model was proposed, namely reversing Petri nets (RPNs)~\cite{PP}. In the newly introduced model, three (motivated by real concurrent systems) computational semantics were considered, namely: backtracking, causal reversing and out-of-causal-order reversing. It was also shown how to encode reversing Petri nets into coloured Petri nets with a finite number of colours, equivalent to the classical bounded place-transition \linebreak systems~\cite{BGMPPP}.

This paper has two \textbf{goals}. The first is to extend the results presented in \cite{BGMPPP}. The new type of history is added
and formal proofs of generation of CPNs from RPNs are presented. We also test the generation on a~number of examples,
where the CPN Tools \cite{CPNtools} have been employed to illustrate that the translations
conform to the semantics of reversible computation.
The second goal is to discuss possibility of the introduction of cycles
to RPNs and how it would impact the reversibility.

\eject
\textbf{Paper organization.} In the following two sections we give an overview
of reversing Petri nets (RPNs) and Coloured Petri
nets (CPNs). Section~\ref{sec:trans_bt} contains description of the generation of CPN based on RPN. This is carried out in two
steps: first CPN mimicking RPN behavior in
forward execution of transitions is obtained, then possibility of reversing is
added to the CPN. Section~\ref{sec.cycles} focuses on introduction
of cycles to RPNs, whether it is possible and how it would impact the
reversing of transitions. The paper is concluded in
Section~\ref{sec.cons}.

\section{Reversing Petri nets}
\label{sec.RPN}

In this section we present the basic concepts of reversing Petri nets (RPNs) based on~\cite{BGMPPP,PP}. In general, the idea of RPNs is to allow reversing computations as easily as the forward ones. Computations in this context mean firing of transitions. Following the biological inspiration (for example catalytic reactions), tokens in RPNs are persistent and distinguishable, and one may associate them with atoms or chemical molecules. The role of transitions is to create bonds between tokens (similar to chemical bonds) or to simply transport them. Reversing of transitions is equivalent to breaking of bonds. Hence, RPNs are naturally suitable to model biological reactions. However, in general, tokens may represent any objects, and bonds any interactions between those objects. An example of RPN is presented in Figure~\ref{figure1}.
\\
\\
\textbf{Preliminaries}\\ \\
The set of non-negative integers is denoted by $\Nat$.
Given a set X, the cardinality (number of elements) of $X$ is denoted by $\#X$, the powerset (set of all subsets) by~$2^X$ -- the cardinality of the powerset is~$2^{\#X}$.
\begin{definition}\label{def:RPN}{\rm
 A \emph{\RPN (RPN)} is a tuple $(P,T,F,A,B)$ where:
\begin{enumerate}
\item $P$ and $T$ are finite sets of \emph{places} and \emph{transitions}, respectively.
\item $A$ is a~finite set of \emph{bases} or \emph{tokens}.  The set $\overline{A} = \{\overline{a}\mid a\in A\}$
contains a ``negative'' instance for every element in $A$\footnote{Elements of $A$ emblem the presence of the base, when elements of $\overline{A}$
the absence of the it. Utilising of the concept can be found in Definition~\ref{def:forward-enabled}.}.
\item $B\subseteq \{\{a,b\}\mid a\neq b\in A\}$ is a set of \emph{bonds}.
We use the notation $a \bond b$ for a bond $\{a,b\}\in B$. The set
$\overline{B} = \{\overline{\beta}\mid \beta\in B\}$ contains a~``negative'' instance for each bond in $B$, similarly as for bases.
\item $F : (P\times T  \cup T \times P)\rightarrow 2^{A\cup\overline{A}\cup B\cup\overline{B}}$ is a set of directed arcs
associated with a subset of $A\cup\overline{A}\cup B\cup\overline{B}$.
\end{enumerate}
}\end{definition}

In the above definition the sets of \emph{places} and \emph{transitions} are understood in the standard way
(see~\cite{PN}).

For a transition $t\in T$ we introduce ${}^\bullet t=\{p\in P\mid F(p,t)\neq\es\}$, $t{}^\bullet=\{p\in P\mid F(t,p)\neq\es\}$
(sets of input and output places of $t$), and
$\guard{t}=\bigcup_{p\in P}F(p,t)$, $\effects{t}=\bigcup_{p\in P}F(t,p)$ (unions of
labels of the incoming/outgoing arcs of $t$),
as well as $\effect{t}=\effects{t} \setminus \guard{t}$. If $\{a,b\} \in B$ and $\{b,c\} \in B$, instead of $a\bond b, b\bond c$ we use the notation $a \bond b \bond c$ (and similar for more bonds).

The following restrictions give rise to the notion of well-formed RPNs.

\begin{definition}\label{def:well-formed-RPN}{\rm
A \RPN $(P,T,F,A,B)$ is \emph{well-formed}, if it satisfies the following conditions for all $t\in T$:
\begin{enumerate}
\item $A\cap\guard{t}=A\cap\effects{t}$,
\item if $a\bond b\in \guard{t}$ then $a\bond b \in \effects{t}$,
\item for every $t\in T$ we have: ${}^\bullet t\neq \es$ and $\#(t {}^\bullet) = 1$,
\item if $a,b\in F(p,t)$ and $\beta=a\bond b\in F(t,q)$ then either $\beta\in F(p,t)$, or $ \overline{\beta}\in F(p,t)$.
\end{enumerate}}
\end{definition}

Clause (1) indicates that transitions do not erase any tokens and clause (2) indicates
that transitions do not destroy bonds. In (3) forks are prohibited in order to avoid
 duplicating tokens that are transferred into different output places but are already bonded
 in the input places.
Finally, clause (4) indicates that tokens/bonds cannot be
recreated into more than one output place --
if a bond appears on the output of a transition, then either that bond have already existed and the transition only transports it (case $\beta\in F(p,t)$), or it is being created and we need to make sure that it has not existed before (case $ \overline{\beta}\in F(p,t)$).
All those clauses are inspired by biological reactions (for example number of atoms is substrates and
products has to be constant).

\medskip
A marking is a distribution of tokens and bonds across places,\\
$M:P\rightarrow 2^{A\cup B}$, where for $p\in P$ if $a\bond b\in M(p)$ then $a,b\in M(p)$.

\medskip
For now we focus only
on acyclic RPNs hence every transition can be executed only once.
However, due to future
assumptions (see Remark~\ref{remarkTwo} related to cycles),
we want to consider transitions in RPNs which could be fired twice.
Because of that, in the paper we would present
definitions and theorems where this fact is already taken into account.

Let $\Nat_2$ be a set containing the empty set, singletons or two-elements sets of natural numbers: i.e. $\Nat_2\subseteq 2^\Nat$ and $\forall_{X\in \Nat_2} \#(X)\leq 2$.
A~\emph{history} assigns an index to each transition occurrence, $H:T\rightarrow \Nat_2$.
An empty-set
history associated with a transition $t\in T$ means that $t$ has not been
executed yet or it has been reversed and not executed again, while a~history of  $\{k_i,k_j\}$
indicates that $t$ was executed as the $k_i^{th},k_j^{th}$ transition in the
computation (and not reversed until this moment).
$H_0$ denotes the initial history where $H_0(t)=\emptyset$ for every $t\in T$.
A \emph{state} is a pair $\state{M}{H}$ of a marking and a history.

Now we introduce the set $\connected(a,C)$ containing $a$ if $a$ is a part of $C$ and a set of tokens connected with $a$ via bonds which are in $C$
as follows

\begin{definition}{\rm
For $a\in A$ and $C\subseteq A\cup B$ we define the following set: \\
$\connected(a,C)=
(\{a\}\!\cap\! C)\cup \{b, c, \{b,c\} \!\mid\! \exists_{w \in 2^B} w=\langle \beta_1, \beta_2, \ldots, \beta_n\rangle, \beta_i \in C\cap B, \beta_i = \{a_{i-1},a_i\}, a_i\in C\cap A,  a_0=a, \beta_n=\{b, c\}, i \in 1,\ldots, n\}$.
}\end{definition}

During biological reactions and other processes, various types of reversing are possible.
In some cases, only the last operation can be reversed (\emph{backtracking}). In other instances, the action can be rollbacked if all its effects have been undone (\emph{causal reversing}), no matter when this action was performed. In the last category of reversing, any previously executed operation can be undone (\emph{out of causal reversing}).
All those three types of reversing are possible in RPNs - only the definition of enableness and mechanism of bonds breaking should be changed to switch between reversing categories.

Note that, in this paper we only focus on backtracking and causal reversing. More information about the third semantics one can find in~\cite{BGMPPP}.

\subsection{Reversing Petri nets - forward execution}\label{sec:semantics}
From now on we assume RPNs to be well-formed. Furthermore, as in~\cite{PP}, we assume that in the initial marking $M_0$ of RPN,
there exists exactly one base of each type, i.e., $\#\{p\in P\mid a\in M_0(p)\} = 1$, for all $a\in A$.
Now we can indicate the conditions that must be met for a~transition of a RPN to be enabled.

\begin{definition}\label{def:forward-enabled}{\rm
Consider a \RPN $(P,T,F,A,B)$, a transition $t\in T$, a state $\state{M}{H}$, a~base $a\in A$, and a~bond $\beta\in B$. We say that
 $t$ is \emph{(forward) enabled} in $\state{M}{H}$ if the following hold:
 \begin{enumerate}
\item  if $a\!\in\! F(p,t)$, resp. $\beta\!\in\! F(p,t)$, for $p\!\in\!{}^\bullet t$, then $a\!\in\! M(p)$, resp. $\beta\!\in\! M(p)$,
\item  if $\overline{a}\!\in\! F(p,t)$, resp. $\overline{\beta}\!\in\! F(p,t)$ for $p\!\in\! {}^\bullet t$, then $a \!\not\in\! M(p)$, resp. $\beta \!\not\in\! M(p)$,
\item if $\beta\!\in\! F(t,p)$ for $p\!\in\! t{}^\bullet$ and $\beta\!\in\! M(q)$ for $q\!\in\! {}^\bullet t$ then $\beta\!\in\! F(q,t)$.
 \end{enumerate}
}\end{definition}

A transition $t$ is enabled in a state $\state{M}{H}$ if all tokens from $F(p,t)$ for every
$p\in {}^\bullet t$ (i.e., tokens required for the firing of the transition) are available, and none
of the tokens whose absence is required exists in an input place of the transition
(clauses 1 and 2). Clause 3 indicates that if a pre-existing bond appears
in an outgoing arc of a transition then it is also a precondition for the transition to fire.

\begin{definition}\label{def:for-eff}{\rm
Given a \RPN $(P,T,F,A,B)$, a state $\state{M}{H}$, and a~transition $t$ enabled in
$\state{M}{H}$, we write $\state{M}{H}
\trans{t} \state{M'}{H'}$
where:
\[
\begin{array}{rcl}
	M'(p) & = & \left\{
	\begin{array}{ll}
		M(p)\setminus \bigcup_{a\in F(p,t)}\connected(a,M(p)),\hspace{0.5in}  & \textrm{if } p\in  {}^\bullet{t} \\
				M(p)\cup F(t,p)\cup \bigcup_{ a\in F(t,p), q\in{}^\bullet{t} }\connected(a,M(q)), & \textrm{if }  p\in t{}^\bullet\\
        	M(p), &\textrm{otherwise}
	\end{array}
	\right.
\end{array}
\]
and
$H'(t')=H(t')\cup \{\max\{k| k \in H(t''), t''\in T\} +1\}$, if $t'=t$,
and $H(t')$ otherwise.
}\end{definition}

After the execution of transition $t$, all suitable (according to Definition~\ref{def:for-eff}) tokens and bonds occurring in its incoming arcs together with elements connected to them by bonds are
transferred from the input places to the output place of $t$. Moreover, the history function
$H$ is changed by assigning the next available integer number to the transition. An example of forward execution of transitions
can be seen in Figure~\ref{figure1}.

\begin{figure}[ht]
\vspace*{1mm}
\begin{center}
\includegraphics[scale=0.53]{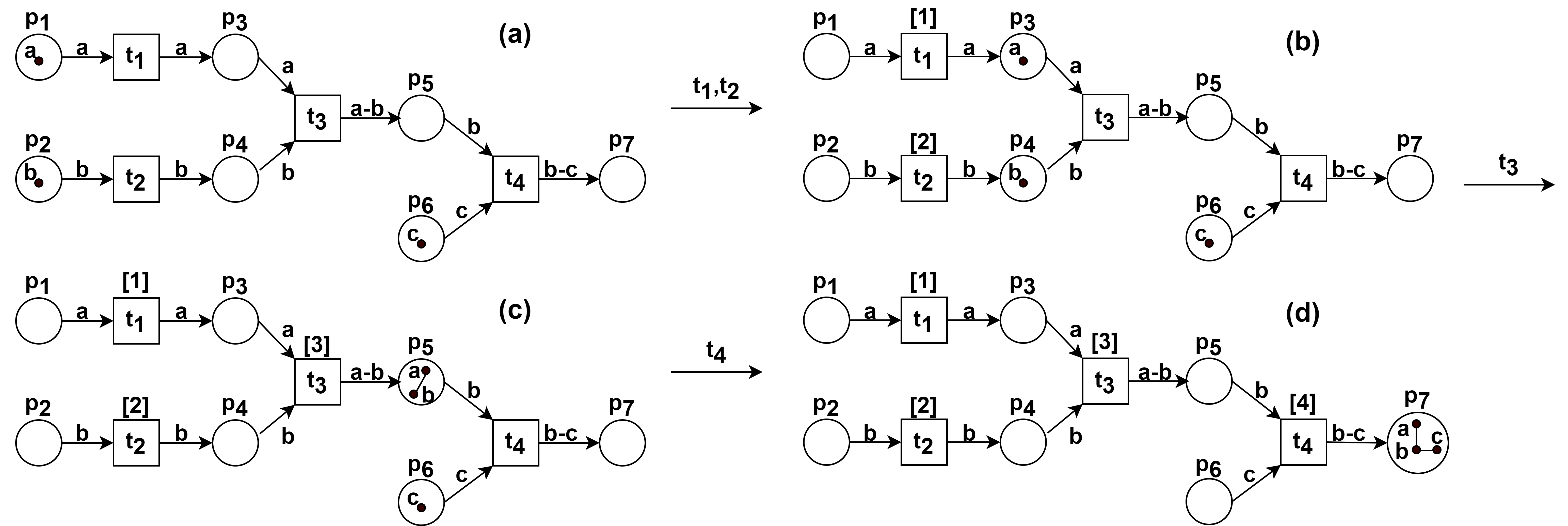}\vspace*{-1mm}
\caption{Example of RPN and its forward execution. Transitions $t_1$
and $t_2$ only transport tokens (token $a$ and $b$, respectively).
Transition $t_3$ requires token $a$ from place $p_3$ and token $b$ from place $p_4$ and it creates a bond between those tokens
($a\bond b $) and transport them to $p_5$. Transition $t_4$ requires token $c$ from $p_6$ and $b$ from $p_5$ -- which is present there after execution of $t_3$. The fact that $b$ is already connected with $a$ is irrelevant for $t_4$ -- it~transports them both together and creates a bond between $b$ and $c$. The whole \emph{molecule} ($a\bond b \bond c$) is transported to $p_7$. Transitions history
is presented as numbers above transitions.}\label{figure1}
\end{center}\vspace*{-3mm}
\end{figure}

In a natural way, we extend the notion of enabledness and transition execution to sequences of transitions:
\newpage
\begin{definition}{\rm Given a RPN  $(P,T,F,A,B)$ and a sequence of transitions
$\sigma = t_1 t_2 \ldots t_n$, where $t_i \in T$ ($i \in 1,\ldots, n$), we say that:
	\begin{itemize}
		\item sequence $\sigma$ is \emph{enabled} in state $\state{M}{H}$ if there exists a sequence of states
$\state{M_1}{H_1},
 \ldots, \state{M_n}{H_n}$ such that
$\state{M}{H} \trans{t_1} \state{M_1}{H_1} \trans{t_2} \ldots \trans{t_n}\state{M_n}{H_n}$,
		\item state $\state{M_n}{H_n}$ is called a \emph{resulting state}, and we write
$\state{M}{H}\trans{\sigma}
\state{M_n}{H_n}$,
	\item a state $\state{M_0}{H_0}$ where for all $t\in T$, $H_0(t)=\emptyset$ is called an \emph{initial state},
	\item a~state $\state{M}{H}$ is \emph{reachable} from the initial state $\state{M_0}{H_0}$
	if there exists a sequence $\sigma$, such that $\state{M_0}{H_0}\trans{\sigma} \state{M}{H}$.
\end{itemize}
}
\end{definition}

We now present the semantics for the various
forms of reversibility as proposed in~\cite{PP}.

\subsection{Backtracking}\label{ssec:backtracking}

A transition is backward enabled (\emph{$bt$-enabled}) if the following holds:

\begin{definition}\label{def:bt-enabled}{\rm
Consider a \RPN $(P,T, F, A, B)$ a state $\state{M}{H}$ and a transition $t\in T$. We say that $t$ is \emph{$bt$-enabled} in
$\state{M}{H}$ if
$k\in H(t)$ with $k\geq k'$ for all $k' \in H(t')$, $t'\in T$.
}\end{definition}

Thus, only the last executed transition can be backward executed in this semantics.
The effect of backtracking a transition in a \RPN is as follows:

\begin{definition}\label{def:bt-eff}{\rm
Given a RPN $N=(P,T,F,A,B)$, a state $\langle M, H\rangle$, and a~transition $t$  that is $bt$-enabled
in $\state{M}{H}$, we write $ \state{M}{H}
\btrans{t} \state{M'}{H'}$
where:
\[
\begin{array}{rcl}
	M'(p) & = & \left\{
	\begin{array}{ll}
		M(p)\cup\bigcup_{a\in F(p,t)\cap F(t,q)}\connected(a,M(q) \setminus \effect{t}),\hspace{0.1in} & \textrm{if } p\in {}^\bullet{t} \\
		M(p)      \setminus \bigcup_{a\in F(t,p)}\connected(a,M(p)), 									& \textrm{if }  p\in t{}^\bullet\\
             M(p), 																				& \textrm{otherwise}
	\end{array}
	\right.
\end{array}
\]
$H'(t')  =  H(t') \setminus \{\max\{k\mid k\in H(t')\}\}$,  	 if $t'=t$,
$H'(t')$, otherwise
}\end{definition}

The crucial element in the reversing is to establish a set of tokens in a given place $p$
 that are connected by bonds to a token $a$ in marking $M$ - this element is described as $\mathsf{con(a, M(p))}$.
When transition $t$ is reversed (in \emph{backtracking} semantic)
we add to its input places elements (tokens and bonds between them) obtained after undoing the effect of $t$, but only those elements which are included in the description of the arc between the input place and  transition $t$ (the first line in the definition).
For the output places of $t$ we remove element (tokens and bonds between them) containing token, which was put there by that transition.
Moreover the history function $H$ of $t$ has to be changed by removing the maximal element of the set to capture that the transition has been reversed.

\begin{example}
In part (d) of Figure~\ref{figure1}, if we decide to reverse $t_4$, a bond between $b$ and $c$ would be broken - because it is an effect of $t_4$. Token $c$ would go back to $p_6$ and element $a\bond b$ to $p_5$ -- it would lead to the marking presented in part $(c)$. Similar situation would occur during reversing of $t_3$ from the marking depicted in part (c). Transitions $t_2$ and $t_1$ have only one input and output place each, hence their reversing would result in transferring token $b$ or $a$ (respectively) from their output to input places.
\end{example}

\subsection{Causal reversing}\label{ssec:causal}

In causal reversing semantic, transition $t\in T$ can be reversed if all transitions dependent on $t$, and
executed after $t$, have been reversed. Therefore, causal enabledness is defined  as follows.

\begin{definition}\label{co-enabled}{\rm
Consider a \RPN $( P,T,F,A,B)$ and a state $\state{M}{H}$. Transition $t$ is
$co$-enabled in  $\state{M}{H}$ if
$H(t)\neq \es$ and for all $t'$ that are \emph{dependent}
on $t$ then either $H(t')=\es$ or $\max(H(t))\geq
\max(H(t'))$.
}\end{definition}

Notice, that behavior of RPN in causal semantics would be determined by the definition of dependence. This will be discussed more in the second part of the paper (Section~\ref{sec.cycles}). So far, we would focus on \emph{the classical} definition:

\begin{definition}\label{structural}{\rm
		Let $(P, T, F, A, B)$ be RPN, and $t_1, t_2\in T$.
		Transitions $t_1, t_2$ are \emph{(structurally) dependent} (we use the notation: $(t_1, t_2) \in Dep$)
		if an input place of one of them is an output place of the other:
		$(t_1, t_2) \in Dep \Rightarrow (t_1^{\bullet} \cap {}^{\bullet}t_2 \neq \emptyset)
		\ \lor \ ({}^{\bullet}t_1 \cap t_2^{\bullet} \neq \emptyset) $.}\footnote{Note that the definition clearly determines the symmetry
		of the relation, i.e.,~$(t_1, t_2) \in Dep \iff (t_2, t_1) \in Dep$.}
	\label{d:dep}
\end{definition}

The effect of causally reversing of transition in \RPN is as follows:

\begin{definition}\label{def:co-eff}{\rm
Given a RPN $N=(P,T,F,A,B)$, a state $\langle M, H\rangle$, and a~transition $t$ with history $k$
$co$-enabled in $\state{M}{H}$, we write $ \state{M}{H}
\ctrans{t} \state{M'}{H'}$
for $M'$ as in Definition~\ref{def:bt-eff} and \\
\[
\begin{array}{rcl}
	H(t') & = & \left\{
	\begin{array}{ll}
		H(t') - \{\max\{k\mid k\in H(t')\}\},\hspace{0.1in} & \textrm{if } t'=t \\
		\{k'  \mid  k' \in H(t'), k' < k\} \cup \{k' - 1 \mid k' \in H(t'), k' > k \},
		& \textrm{otherwise}
	\end{array}
	\right.
\end{array}
\]
}\end{definition}

In many cases reversing according to the backtracking and causal semantic
would be the same.
\begin{example}
In Figure~\ref{figure1} part $(d)$ in both semantics
only transition $t_4$ could be reversed. It would lead to the state presented in part~$(c)$ of the figure. Then, once again, only one transition could be reversed - transition $t_3$ and it would lead to the marking presented in part~$(b)$. At this point we can see the difference between both semantics. In backtracking, transition $t_2$ has to be reversed before transition $t_1$, because their were fired in that (opposite) order. For causal reversing,
both transitions could be reversed, because they are not dependent. Hence, transition $t_1$ could be reversed before $t_2$, even if in forward execution $t_1$ was fired before $t_2$.
\end{example}

\subsection{Returning to the initial state}

The following theorem states that starting from the initial marking and executing a sequence of transitions we
may return back (using backtracking or causal reversing semantics) to the initial marking if all the forward transitions
are reversed. Let $\xmapsto{\sigma}$ denotes
  a sequence of both forward and reversed transitions.
Moreover, for a sequence $\sigma\in (T \cup \{\underline{t} \mid t \in T \})^*$, the occurrence of $t$, written without underlining, means that
transition $t\in T$ was executed in the forward direction in $\sigma$, while the occurrence of~$\underline{t}$, underlined, indicates
that transition $t\in T$ was executed in the reverse direction.

\begin{theorem}\label{theorem1}{\rm
	If $\state{M}{H}\xmapsto{\sigma}\state{M'}{H'}$ where $\{t|t\in \sigma\}=\{t|\underline{t} \in \sigma \}$ then $M=M'$ and $H=H'$.
}\end{theorem}

\paragraph{Proof of Theorem~\ref{theorem1}:}
Suppose that $\state{M}{H}\xmapsto{\sigma}\state{M'}{H'}$ where $\sigma \in T^*$ is a sequence of forward
and reverse transitions such that $\{t|t\in \sigma\}=\{t|{\underline{t} \in \sigma}\}$.
We prove the theorem by induction on the length $n$ of $\sigma$ ($n=|\sigma|)$.
If $n=0$, there are no transitions in $\sigma$ and the theorem is trivially satisfied. If not, we assume that the theorem holds
for $k<n$ and proceed by induction.
Let $t$ be the first transition in the sequence executed in the reverse direction.
We distinguish two cases:
\begin{enumerate}
	\item If the pair of transitions $t \underline{t}$ constitutes a factor of the sequence $\sigma$,
	then we can remove $t\underline{t}$ from $\sigma$. This operation is correct because reversing $t$ just
	after its execution undoes the effect of $t$ and leads to the marking before execution of $t$.
	This way we obtain a shorter sequence $\sigma'$,
	which is equivalent to the former one (in the meaning of reachable markings).
	Since $|\sigma'|<|\sigma|$
	the proof follows by induction.
	Note that this part holds both for backtracking and co-reversing.
	\item If the pair of transitions $t'\underline{t}$ (for $t'\neq t$) constitutes a factor of the sequence $\tau$,
	then it means that for $\underline{t}$ to be executed (strictly speaking: for $t$ to be reversed)
	it must be that $t$ has been executed before $t'$.
	Note that in backtracking semantics, this situation is impossible, as reversing can only happen immediately after the execution of transition $t$,
	hence this part is crucial only for causal-order reversing semantics.
	By~Definition~\ref{co-enabled}, all transitions located in the sequence $\sigma$ between $t$ and $\underline{t}$ are\
	 independent of $t$ (if not, it would not be possible for $t$ to be reversed before their reversal and we assume that
	 $t$ is the first occurrence of a reverse transition).
	As a result $t$ can be swapped with all of them, resulting in a new equivalent sequence containing $t \underline{t}$.
	In this situation, the previous case applies.
\end{enumerate}
This completes the proof.
\proofend

\section{Coloured Petri nets}
\label{sec:CPN}
Recall that RPNs constitute a model in which transitions can be reversed
according to three semantics: backtracking, causal, and out-of-causal-order
reversing. A main characteristic of RPNs is the concept of a \emph{history}, which assigns
a set of natural numbers to transitions.
However, it imposes
the need of a global control in order to reverse computations. Our goal is
to recast the model of RPNs into one without any form of global control
while establishing the expressiveness relation
between RPNs and the model of bounded coloured Petri nets.
In this section we recall the notion of coloured Petri nets (CPNs).

Note that, according to the utilised CPN Tools \cite{cpn-tools},
$\mathit{EXPR}_V$ is the set of \emph{net inscriptions} (over a~set of variables $V$, possibly empty,
i.e., using only constant values) provided by CPN ML.
Moreover, by $\mathit{Type}[e]$ we denote the type of values obtained by the evaluation
of expression $e$.
The set of \emph{free variables} in an expression $e$ is denoted by $Var[e]$.
The setting of a particular value to free variable $v$ is called~a \emph{binding}
$b(v)$. We~require that $b(v)\in \mathit{Type}[v]$
and denote with the use of $\langle\rangle$ filled by the list of valuations and
written next to the element to whom it relates.
The set of bindings of $t$ is denoted by $B(t)$.
The \emph{binding element} is a~transition $t$ together with a valuation $b(t)$
of all the free variables related to $t$.
We~denote it by $(t,b)$, for $t\in T$ and $b\in B(t)$.

\begin{definition}[\cite{CPN}]{\rm
\label{def:CPN}
A (non-hierarchical) \emph{coloured Petri net} is a nine-tuple
$CPN=(P,T,D,\Sigma,V,C,G,E,I)$, where:
\begin{itemize}
\item $P$ and $T$ are finite, disjoint sets of \emph{places} and \emph{transitions};
\item $D\subseteq P\times T\cup T\times P$ is a set of \emph{directed arcs};
\item $\Sigma$ is a finite set of non-empty \emph{colour sets};
\item $V$ is a finite set of \emph{typed variables} such that $Type[V]\in\Sigma$ for all $v\in V$;
\item $C:P\rightarrow\Sigma$ is a \emph{colour set function} that assigns colour sets to places;
\item $G:T\rightarrow \mathit{EXPR}_V$ is a \emph{guard function} that assigns a guard to each transition $t$ such that $\mathit{Type}[G(t)]=Bool$;
\item $E:D\rightarrow \mathit{EXPR}_V$ is an \emph{arc expression function} that assigns an arc expression to each arc $d\in D$ such that $\mathit{Type}[E(d)]=\Nat^{C(p)}$, where $p$ is the place connected with the arc $d$;
\item $I:P\rightarrow \mathit{EXPR}_\emptyset$ is an \emph{initialisation function} that assigns an initialisation expression to take each place $p$ such that $\mathit{Type}[I(p)]=\Nat^{C(p)}$.
\end{itemize}}
\end{definition}

A \emph{marking} $M$ in coloured Petri nets is a function which assigns
a set of tokens $M(p)\in 2^{C(p)}$ to each $p\in P$.
An initial marking is denoted by $M_0$ and defined for each $p\in P$ as follows:
$M_0(p) = I(p)\langle\rangle$.

A binding element $(t,b)$ is \emph{enabled} at a marking $M$ if $G(t)\langle b\rangle$ is true and
at each place $p\in P$ there are enough tokens in $M$ to fulfil
the evaluation of the arc expression function $E(p,t)\langle b\rangle$.
The resulting marking is obtained by removing the tokens given by $E(p,t)\langle b\rangle$ from $M(p)$
and adding those given by $E(t,p)\langle b\rangle$ for each $p\in P$.

We define the \emph{enabledness} of transition in CPN as follows:
a transition $t \in T$ is \emph{enabled} at $M$ and its execution
leads to marking $M'$ (denoted $M[t\rangle M'$) if there exists a binding $b \in B(t)$,
such that the binding element $(t, b)$ is enabled at $M$.

\section{Generation of CPN from RPN}
\label{sec:trans_bt}

In this section we describe how to create CPN corresponding to a given RPN. The process is divided into two steps: in the first
we present how to generate CPN only for the structure of RPN and forward execution semantic, without implemented reversing semantics (Section~\ref{sec:str} and Section~\ref{sec:trans_co}). In the second the reversing semantics are added to CPN in a~form of additional transitions and arcs
(Section~\ref{sec:rev}).

\subsection{Generation of CPN - the structure and forward executions}
\label{sec:str}

We design the transformation of RPN $N_R=(P_R, T_R, F_R, A_R, B_R)$
to a~new equivalent CPN $C_R=(P_C,T_C,D_C,\Sigma_C,V_C,C_C,G_C,E_C,I_C)$ as
follows.

\medskip
The set of places is $P_C = P_R \cup P_{THP} \cup P_{CHP}$, where:
\begin{itemize}
\item $P_R$ is a set of places from the original RPN $N_R$,
\item $P_{THP} = \lbrace h_i \mid t_i \in T_R\rbrace$ is a set of \emph{transitions history places} (one new place for every transition from the original net),
\item $P_{CHP} = \lbrace h_{ij} \mid t_i, t_j \in T_R, i < j\rbrace$ is a set of \emph{connection history places} (one new place for every pair of transitions from the original net).
\end{itemize}

The set of transitions of the net $C_R$ is the same as in the RPN, namely $T_C = T_R$.
The set of variables $V_C$ should contain all elements necessary to describe each input token of a transition.

\medskip
New arcs have to be added to
$C_R$ to connect newly added places. Each transition $t_i$ is connected with its
history place $h_i$\footnote{Whenever the denotation $h_i$ is used without explanation,
we assume this is a~transition history place for transition $t_i \in T_R$.}
and all its connection history places ($h_{ij}$ or $h_{ji}$, depending on the order of $i$ and $j$,
where $j$ is a~number of transition, different from $i$)
in both directions. Hence:
\begin{align*}
D_C =\; & Domain(F_R) \cup \lbrace (t_i, h_i) \mid {t_i \in T_R} \rbrace
		\cup \lbrace (h_i, t_i) \mid {t_i \in T_R} \rbrace \\
	& \cup \lbrace (t_i, h_{ij}) \mid {t_i \in T_R}, i < j \rbrace
		\cup \lbrace (t_i, h_{ji}) \mid {t_i \in T_R}, j < i \rbrace \\
	& \cup \lbrace (h_{ij}, t_i) \mid {t_i \in T_R}, i < j \rbrace
		\cup \lbrace (h_{ji}, t_i) \mid {t_i \in T_R}, j < i \rbrace.
\end{align*}

The set of colours $\Sigma_C$ contains:
\begin{itemize}
\item $\mathit{Base} = A_R$;
\item $\mathit{Bond} = B_R$;
\item $\mathit{Bases}$ (subsets of $\mathit{Base}$ - in CPN Tools represented as lists);
\item $\mathit{Bonds}$ (subsets of $\mathit{Bond}$ - in CPN Tools represented as lists);
\item $\mathit{Molecule} = \mathit{Bases} \times \mathit{Bonds}$ -- molecules, as in a biochemical system, are considered to be a~set of bases or atoms with the corresponding bonds between them;
\item $\mathit{HIST} = \lbrace (n, i, j) \mid i, j, n \in \mathbb{N} \rbrace$ (local history for a pair of transitions) and
\item $\mathit{boundInt}$ -- bounded natural numbers belonging to $\Nat_b$ (the bound is equal to $\#T_R \cdot 2$).
\end{itemize}

The colour function $C_C$ assigns:
\begin{itemize}
\item to every place ${p \in P_R}$ -- a $molecule$ colour;
\item to every connection history place $h_{ij} \in P_{CHP}$ -- $\mathit{boundInt}$ colour,
which is a bounded integer number which describes how many times transitions from the pair $t_i, t_j$ were executed;
\item to every transition history place $h_i \in P_{THP}$ -- $\mathit{HIST}$ colour \footnote{If a triple $(n, j, i)$ is present in place $h_i$, it means
that transition $t_i$ occurred at the $n^{th}$ position in a sequence of executions of transitions $t_i$ and $t_j$.}.
\end{itemize}

The guard function $G_C$ has to be equivalent to the labels of input arcs defined in the RPN $P_R$. Consequently, if
$a \in F_R(p, t_i)$ ($\beta \in F_R(p, t_i)$, respectively) for a transition $t_i$ and its input place $p$, then
$G_C(t_i)$ should contain a condition, assuring that the binding of an
input token for place $p$ contains $a$ ($\beta$,~respectively).

The arc expression function $E_C$ for arcs between transitions $t_i \in T_R$ and places
$p_i \in P_R$ should be analogous to $F_R(t_i, p_i)$. If $t_i$ only transfers tokens
then $E_C(t_i, p_i)$ should be a union of bonds and bases of all inputs for $t_i$.
If $t_i$
creates a bond $\beta$, then $E_C(t_i, p_i)$ should be an union of bonds and bases of all inputs
for $t_i$, together with the newly created bond $\beta$.

Places $h_{ij}$ and $h_i$ control the history of a transition $t_i$,
(here, without lost of generality, we can assume that $i < j$). Let
$history_{ij} \in V_C$ represents the value obtained from place $h_{ij}$ by $t_i$.
The arc expression function is defined as:
${E_C(h_{ij}, t_i) = history_{ij}}$, $E_C(t_i, h_{ij}) = history_{ij} + 1$.
Hence, the current value of the connection history place
$h_{ij}$ denotes the next history value for the pair of transitions $t_i$ and $t_j$.

For the transition history place $h_i$, the following arc expressions should be assigned:
$E_C(h_i, t_i) = list_i$, where $list_i$ is a list of triples, which describes the previous history
of the transition $t_i$ and
$E_C(h_i, t_i) = list_i
\cup \lbrace (history_{ij}, j, i) \mid t_j \in T_R, M(h_{ij}) = history_{ij} \rbrace$. Understanding the history mechanism is crucial for understanding
the transformation idea.
Since we assumed that
each path in RPN is finite, values in places $h_i$ and $h_{ij}$ are bounded by the definition.

The initialization function $I_C$ may be understood as an assignment of the initial marking
to places. From now on by \emph{markings} we understand the value of tokens in places (according to definitions of CPNs the concept of marking is more complex, hence this statement).
For places $p \in P_C$ originated from $P_R$ we assign the same initial marking (the same set of bases and bonds)
as in the original net.
For $h_i \in P_{THP}$ we have $I_C(h_i) = \es$ (empty list), while for
$h_{ij} \in P_{CHP}$ we have $I_C(h_{ij}) = 0$.

\medskip
Let us now define the correspondence between states of RPN and markings of the corresponding CPN.
First, recall that in acyclic RPNs transitions may be fired at most once (because every base or bond appears only once
in any marking), but in Section~\ref{sec.cycles} we discuss
transitions which may be executed twice, hence here we would already present result with this assumption. Recall that:
\begin{itemize}
\item a state in RPN is a pair $\langle M_R, H_R \rangle$, where
$M_R:P \rightarrow 2^{A \cup B}$ and $H_R: T \rightarrow \Nat_2$,
\item a state in CPN generated from RPN can be considered as a marking $M:P \rightarrow 2^{C_C(p)}$.
\end{itemize}

\begin{remark}
A marking in RPN is a set of bases and bonds, while a marking in CPN for places originating from RPN
is a set of pairs of the form $(bases, bonds)$. Of course, one representation can
be easily transformed to the equivalent one.
\end{remark}

In what follows we describe how to generate a marking $M$ of CPN on the basis of $\langle M_R, H_R \rangle$ of RPN
or how to obtain the original state $\langle M_R, H_R \rangle$ of RPN from $M$ of CPN. Such marking $M$ and
state $\langle M_R, H_R \rangle$ are called \emph{corresponding}.

\medskip
A marking $M$ of CPN generated from $\langle M_R, H_R \rangle$ of RPN is a function as follows:
\begin{itemize}
\item $M(p) \in 2^{A \cup B}$ for $p \in P_R$ - if a base belongs to $M_R(p)$ then it belongs to
the first coordinate of $M(p)$, and if a bond belongs to $M_R(p)$ then it belongs to
the second coordinate of $M(p)$,
\item $M(h_i) \in 2^{(\Nat \times \Nat \times \Nat)}$, for $h_i \in P_{THP}$ where \\
$M(h_i) = \bigcup_{k \in H_R(t_i); t_i, t_j \in T; i \neq j} (\#\{h \in H_R(t_i) \cup H_R(t_j); h < k \} + 1, j, i)$,
\item $M(h_{ij}) \in 2^{\Nat}$, for $h_{ij} \in P_{CHP}$ where $M(h_{ij}) = \#H_R(t_i) + \#H_R(t_j)$.
\end{itemize}

On the other hand, having a marking $M$ of CPN indicates the original state $\langle M_R, H_R \rangle$ of RPN
in the following way:
\begin{itemize}
\item $M_R(p) = \bigcup_{(x,y) \in M(p)} (x \cup y)$, for $p \in P$;
\item As mentioned before, we assume that transitions in \RPN can be executed  at most twice.
For that reason, for any $t_i \in T$ we can distinguish three cases of the content of transition history
place $h_i$:
\begin{enumerate}
\item $\#M(h_i) = 0$ (i.e., transition $t_i$ has not been executed yet); \\
in that case $H_R(t_i) = \es$.
\item $\#M(h_i) = \#T_R-1$ (i.e., $t_i$ has been executed once); \\
in that case
$H_R(t_i) = \{1 + \Sigma_{(k, j,i) \in M(h_i)} (k-1)\}$
\item $\#M(h_i) = 2 \cdot (\#T_R-1)$ (i.e., $t_i$ has been executed twice); \\
in that case let us define sets: \\
$\mathit{maxHist}(h_i) = \{(k_m, j, i) \in M(h_i) | k_m = max\{k|(k,j,i) \in M(h_i)\}\}$ - for each pair of
 selected transition $t_i$ and every other transition $t_j$
 we choose only the triple with the maximal value of $k$.
\\
$\mathit{minHist}(h_i) = M(h_i) \setminus \mathit{maxHist}(h_i)$\\
Considering above formulas the history of any transition $t_i$,
 which was fired twice would contain:\\
$H_R(t_i) = \{1 + \Sigma_{(k, j,i) \in minHist(h_i)} (k-1)$,\\
$ 1 + \Sigma_{(k_j, j,i) \in maxHist(h_i)} (k_j-1)-
\#\{(k_g, j, i) \in M(h_i) | k_g < k_j \land (k_j, j, i)\in M(h_i)\} * \frac{\#T_R - 2}{\#T_R - 1}\}$.
\end{enumerate}
\end{itemize}

The case, when a transition has not been executed yet is trivial.
In other cases we have to calculate the index of the transition in the sequence of executions. However, this number in CPN is not given
directly, but is scattered among history places, or more precisely among factors $k$ in triples stored in those history places.
When transition $t$ is executed for the first time, a triple is added to its history place for every other transition $t'\in T_R$.
Factor $k$ in such a triple means that the discussed transition was executed as $k$-th when you consider only transitions from the set $\{t,t'\}$. Hence, if the transition $t$ is $k$-th - it means that the other transition ($t'$) has been executed $k-1$ times earlier or in other words, there have been $k-1$ executions before the discussed transition fired. To calculate the index of the transition in the whole sequence, all executions before the considered one have to be added plus one for the considered execution. It gives the formula presented above (case $2$).
Similar situation occurs when the transition is executed for the second time (case 3).
However, in that case there are two groups of triples in the history place. A part of them is related to the first execution (those
belong to $minHist(h_i)$) and can be used to calculate the index in the sequence of executions related to the first execution of the transition.
Others are related to the second execution - those belong to $maxHist(h_i)$. However, we cannot
simply add $k-1$ executions before the considered one as it was described in case $2$, because values $k$ in triples from
$maxHist(h_i)$ contain also information about the first execution (after the execution, counters are not reset). Hence, we need to subtract that redundant information.

\begin{theorem}{\rm
Consider RPN $R=(P_R, T_R, F_R, A_R, B_R)$ and the corresponding CPN \\
 $C=(P_C, T_C, D_C, \Sigma_C, V_C, C_C, G_C, E_C, I_C)$  constructed
according to the above transformation.
Let $\state{ M_R}{H_R}$ be a reachable state in $R$ and $M$
be a corresponding marking in $C$.
Then a~transition $t_i$ is enabled at $M_R$ in $P_R$ if and only if it is enabled at $M$ in $C$.
Moreover, if $\state{ M_R}{H_R} \trans{t_i}\state{ M_R'}{H'_R}$ and $M[t_i\rangle M'$
then $\state{ M_R'}{H'_R}$ corresponds to $M'$.
\label{theoremMain}}
\end{theorem}

\begin{proof}
Let $\state{ M_R}{H_R}$ be a reachable state in $R$ and $M$
be the corresponding marking in $C$.
The enabledness of transitions (in the forward direction) depends only on the molecules located in its
input places (which in coloured Petri net $C$ is expressed by the guard function).
The correspondence between RPN and CPN described above, assumes that the content
of such places is equivalent. Hence, transition $t_i$ is enabled at $\state{ M_R}{H_R}$ if and only if
$t_i$ is enabled at $M$.

Let $\state{ M_R}{H_R} \trans{t_i}\state{ M_R'}{H'_R}$ and $M[t_i\rangle M'$. We need to show that
$\state{ M_R'}{H'_R}$ corresponds to~$M'$. According to the definition of the effect in RPN and
the transformation procedure, we know that the contents of places belonging to $P_R$ before
and after the firing of $t_i$ in $R$ and $C$ are equivalent.
We only need to focus on histories $H_R$ and $H'_R$ in $R$ and markings of transition and connection history
places in $C$.
We know that after execution of $t_i$ in RPN, the new element, indicating the number of transitions executed in
the current computation, is added to its history. It is the only difference between $H_R$ and $H'_R$.
On the other hand, the difference between $M$ and $M'$ considering only history places is as follows:
\begin{itemize}
\item All tokens (i.e., natural numbers) in connection history places related to $t_i$ are increased by 1,
but those places
are not considered during computation of the corresponding state in RPN.
\item New elements, in the number of $\#T_R - 1$, are added to the transition history place $h_i$ of $t_i$.
\end{itemize}
Let us notice, that due to our future assumption (see Remark~\ref{remarkTwo}), for a~given transition
$t_i$ the set $H_R(t_i)$ can be an empty set, a singleton or a~two-elements set. In the last case,
the transition cannot be forward executed any more. Hence, we consider two cases:
\begin{enumerate}
\item $H_R(t_i) = \emptyset$, then after the execution of $t_i$ we have
$H'_R(t_i) = \{l_1\}$, for some natural number $l_1$ indicating the index of the transition in the current
computation, hence $l_1 - 1$ equals to the number of transitions executed before $t_i$ in the sequence.
On the other hand, based on the fact that $M$ corresponds to state $\state{ M_R}{H_R}$, we have
$M(h_i) = \emptyset$. After the execution of $t_i$ at $M$ in $C$, we add $\#T_R - 1$ triples
$(k,j,i)$ to $h_i$ to obtain $M'(h_i)$. From every such triple, based on $k$, we can deduce whether some transition
$t_j$ has been executed
before $t_i$. We only need to count such transitions and add $1$ to obtain $l_1$.
Strictly speaking, we use the formula: $H'_R(t_i) = \{1 + \Sigma_{(k, j,i) \in M'(h_i)} (k-1)\}$.
\item  $H_R(t_i) = \{l_1\}$ for $l_1 \in \Nat \setminus \{0\}$, then after the execution of $t_i$ we have
$H'_R(t_i) = \{l_1, l_2\}, l_2 > l_1$, for some natural number $l_2$ indicating the second index of the transition in the current
computation. On the other hand, based on the fact that $M$ corresponds to state $\state{ M_R}{H_R}$, we have
$M(h_i)$ consisting of $\#T_R - 1$ elements. After the execution of $t_i$ at $M$ in $C$, we add new triples
$(k,j,i)$ to $h_i$, obtaining $M'(h_i) = M(h_i) \cup X$, where $\#X = \#T_R - 1$.
Similarly to the previous case, from every triple $(k,j,i)$, based on $k$, we can deduce whether some transition
$t_j$ has been executed before $t_i$. However, there are two triples in $M'(h_i)$
for every transition $t_j$. We do not want to double the information, hence based on the
inclusion-exclusion principle, we use the following formula to compute $l_2$:
$1 + \Sigma_{(k_j, j,i) \in maxHist(h_i)} (k_j-1)-
\#\{(k_g, j, i) \in M'(h_i) \mid k_g < k_j \land (k_j, j, i)\in M'(h_i)\} \cdot \frac{\#T_R - 2}{\#T_R - 1}$,
where $\mathit{maxHist}(h_i)$ is defined above.
\end{enumerate}

\vspace*{-7mm}
\end{proof}

\subsection{Generation of CPN -- modification for causal-order reversing}
\label{sec:trans_co}
The construction described to this point requires a refinement for
causal-order reversing.
Based on the  structural dependence approach,
as described in Section~\ref{ssec:causal}, two transitions are said to be dependent if an input place of one of them is an output place of the other (see Definition~\ref{d:dep}).  In  order to implement this form of dependence we need a~different approach to define the set $P_C$ than the one used for backtracking.
We still create the same transition and connection history places like described in Section~\ref{sec:trans_bt}.
However, during the evaluation of the guard function for reversing,
we consider instead of the $P_{CHP}$ its subset, called $P_{SHP}$ defined:
$P_{SHP} =  \lbrace h_{ij} \mid t_i, t_j \in T_R; t_i, t_j \in Dep; i < j \rbrace$. All~assignments connected to places in
$P_{SHP}$ stay the same like in $P_{CHP}$.

\subsection{Generation of CPN -- adding reverses}
\label{sec:rev}

The coloured Petri net $C_R$ described in Section~\ref{sec:str}
is prepared for reversing.
This can be achieved by adding supplementary reversal transitions.
The new CPN
$C_R\textprime=(P_C,T_C\textprime,D_C\textprime,\Sigma_C,V_C\textprime,C_C,G_C\textprime$, $E_C\textprime,I_C)$ is
based on $C_R$ (which is CPN corresponding to RPN $P_R)$. The set of places $P_C$ and colours $\Sigma_C$,
the function $C_C$ and the initialization expression $I_C$
are the same as in $C_R$.

For every transition in $C_R$, a new reversed transition $tr$ is added to the net.
Hence, $T_C\textprime = T_C \cup \lbrace tr_i \mid t_i \in T_R \rbrace$.
The execution
of $tr_i$ is equivalent to a rollback of an execution of $t_i$
corresponding to~$tr_i$.

Each transition $tr_i$ is connected to the same set of places
as $t_i \in T_C$ but in opposite directions. Moreover $tr_i$ is connected with
all transition history places and connection history places related to it, namely
$D_C\textprime = D_C \cup \lbrace (tr_i, p) \mid (p, t_i) \in D_C \rbrace
\cup \lbrace (p, tr_i) \mid (t_i, p) \in D_C \rbrace$
$\cup \lbrace (tr_i, h_j) \mid {t_i, t_j \in T_R} \rbrace$
$\cup \lbrace (h_j, tr_i) \mid {t_i, t_j \in T_R} \rbrace \cup
	\lbrace (tr_i, h_{ij}) \mid {t_i \in T_R}, i < j \rbrace
		\cup \lbrace (tr_i, h_{ji}) \mid {t_i \in T_R}, j < i \rbrace
	 \cup \lbrace (h_{ij}, tr_i) \mid {t_i \in T_R}, i < j \rbrace
		\cup \lbrace (h_{ji}, tr_i) \mid {t_i \in T_R}, j < i \rbrace$.

The set of variables $V_C\textprime$ should contain all elements
necessary to describe the input tokens of all transitions (including reversed transitions).

The guard function $G_C\textprime$ has to be modified to take into account the newly
created reversing transitions.
Hence, guards contain conditions checking whether the input places of $tr_i$, which are
originally from $P_R$, contain bases transferred by $t_i$
(for transition which only transfer molecules) or bonds created by $t_i$ (for transition which creates a bond).
Moreover, the conditions used in the guard function for transitions $tr_i$,
in the case of backtracking,
have to guarantee that the transition $t_i$ was the last one executed in a system.
On the other hand, in the case of causal-order reversing, the guard has to assure that no transition dependent
on $t_i$ was executed after $t_i$.
Let $t_i \in T_C$ be the transition to be reversed and $tr_i \in T_C\textprime$ -- its  reverse.
To define the guard function for $tr_i$, we have to look through the content of
the transition history place $h_i$ and the connection history places $h_{ij}$ and $h_{ji}$ for $i \neq j$. For transparency in the following paragraph, having fixed $i$, we use the denotation $h_{ij}$,
regardless of the actual order of $i$ and $j$
(i.e., $h_{ij}:=h_{ij}$ if $i<j$ and $h_{ij}:=h_{ji}$ if $i>j$).

\medskip
For every pair $t_i, t_j$ ($i$ -- fixed, $i \neq j$) we proceed as follows:
\begin{enumerate}
\item Let $history_{ij}$ be the value obtained from place $h_{ij}$.
\item Let $list_i$ be the value obtained from place $h_i$.
\item We check whether $list_i$ contains the element $(history_{ij}, j, i)$. If~`yes',
it means that transition $t_i$ was the most  recently executed one from the pair $t_i, t_j$.
\item If the answer for at least one $t_j$ is `no' then the guard function returns value $\mathit{false}$, otherwise it~returns $\mathit{true}$.
\end{enumerate}

Due to the different definitions of backtracking and
causal-order reversing, in the procedure described above, we use different sets of connection history places
$\{h_{ij} | i\mathrm{-fixed}; i \neq j\}$.
For backtracking we use the whole $P_{CHP}$, for casual reversing set $P_{SHP}$ defined in Section~\ref{sec:trans_co}.

The arc expression function $E_C\textprime$ for arcs between transitions $tr_i$ and places
$p_i \in P_R$ describes the reversal of the execution of $t_i$. Hence,
if $t_i$ just transfers tokens, then the arc description
contains only the transfer of molecules. If $t_i$ creates a bond $\beta = a\bond b$
then during the execution of $tr_i$ the bond has to be broken.
This may result in the production of two separate molecules, one of them including $a$ while the other one including $b$.
Hence, $E_C(p_i, tr_i)$ contains only the transfer of a molecule, and
$E_C(tr_i, p_j)$ contains the transfer of a molecule obtained after breaking
bond $\beta = a\!-\!b$, which includes $a$ ($b$ respectively) if $a$ ($b$ respectively)
has been transferred from place $p_j$ during execution of $t_i$.
If the molecule is still a connected component after breaking the bond, then $E_C(tr_i, p_j)$ indicates the transfer of the whole molecule back to the place from which it was taken by $t_i$
(this situation is possible only when $t_i$ has one input place).
All of these computations can be done using the CPN semantics in combination with the use of functions,
allowed in CPN ML and graph operations.

The arc expression function for the pair $(h_i,tr_i)$ (i.e., $E_C(h_i, tr_i)$)
 allows collecting execution history of the $t_i$, presented as
a list of triples.
For arcs in the opposite direction the arc expression function returns the list without
 elements $(n, j, i)$ for $t_j \in T_R$ describing the last execution of~$t_i$.

For the other transition history places $h_j$, $i \neq j$, the arc expression
$E_C\textprime\textprime(h_{j}, tr_i)$ includes the transfer of a token from $h_{j}$.
The expression $E_C\textprime\textprime(tr_i, h_{j})$
contains the modification of the token value.
We consider the triple $(n,j,i)$ from $h_i$ and selected triples $(m, i,j)$ from $h_j$,
all those triples determined by
the last execution of transition $t_i$ (the one to be reversed).
If $m$ is larger than $n$, the arc expression
$E_C\textprime\textprime(tr_i, h_{j})$ exchanges the value of triple $(m, i, j)$
by $(m-1, i, j)$. No matter whether the value of the token
is modified or not, it needs to be transferred back to the place $h_j$.
Note that, in the case of backtracking, such  modification never happens (because in backtracking
we can reverse only the recently executed transition, hence $m < n$).

\begin{figure}[!ht]
\vspace*{-1mm}
\begin{center}
\includegraphics[scale=0.51]{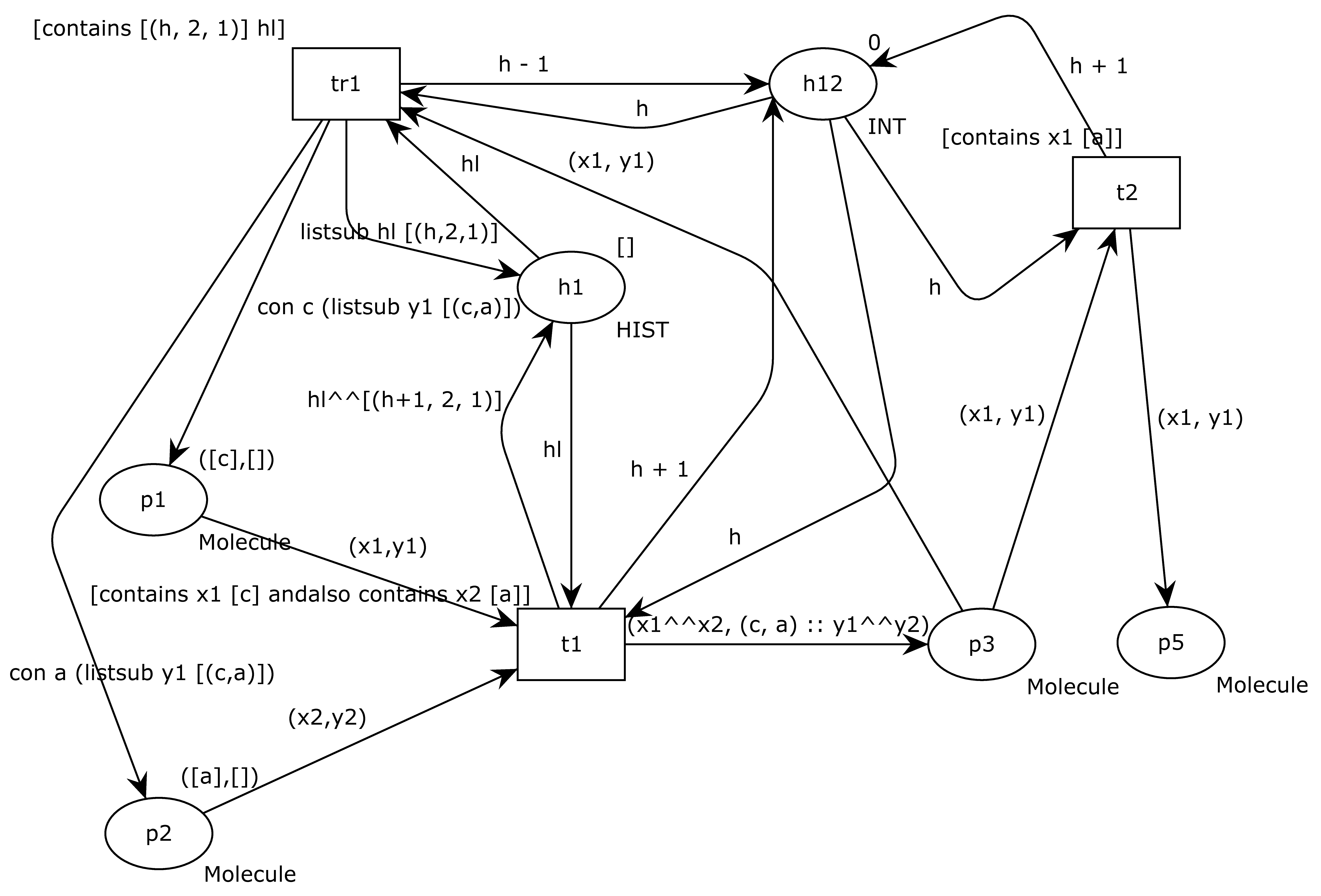}\vspace*{-1mm}
\caption{Example of CPN generated for RPN. }
\label{figure2}
\end{center}\vspace*{-5mm}
\end{figure}

\begin{example}
Figure~\ref{figure2} depicts an example of CPN generated from RPN. For legibility place $h_2$ and transition $tr_2$ is omitted.
Transition $t_1$ creates the bond $a\bond c$, transition $t_2$ transports base $a$. The operations described
in the transformation are implemented as functions in CPN Tools semantics: \smallskip
\texttt{\scriptsize{
\\ \textbf{fun} nei [] x = [x] | nei ((y,z)::xs) x = if y=x orelse z=x then [y,z] ^^ (nei xs y) ^^ (nei xs z) else nei (if ia xs x then xs ^^ [(y,z)]) else []) x
\\ \textbf{fun} cbs x [] = []| cbs x ((y,z)::yr) = if x=y orelse x=z then [(y,z)] ^^ cbs x yr else cbs x yr
\\ \textbf{fun} cb [] [] = [] | cb (x::xr) [] = [] | cb [] l = [] | cb (x::xr) l = cbs x l ^^ cb xr l
\\ \textbf{fun} con x l = (remdupl (nei l x), remdupl (cb (nei l x) l)).
}}\smallskip \\
More elaborated examples (in form of high-resolution images and CPN Tools files) are available here:
\href{https://www-users.mat.umk.pl/~leii/cycles/}{\textcolor{blue}{https://www-users.mat.umk.pl/$\sim$leii/cycles/}}.
\end{example}

When it comes to the connection history places $h_{ij}$
($h_{ji}$ respectively), the
values in those places are decreased by one during the execution of $tr_i$, and
this operation has to be indicated by the arc expression function.

Finally we are ready to prove the correctness of the transformation for reversed executions.
\newpage
\begin{theorem}{\rm
Consider RPN $R=(P_R, T_R, F_R, A_R, B_R)$ and the associated CPN\\
 $C=(P_C, T_C, D_C, \Sigma_C, V_C, C_C, G_C, E_C, I_C)$  constructed on the basis of $P$ according to the transformation described in Section~\ref{sec:rev}.
Let $\state{ M_R'}{H'_R}$ be a reachable state in $R$ and $M'$
be a corresponding marking in $C$.
Then a~transition $t_i$ is enabled in the reverse direction (according to backtracking or causal order semantics)
at $M'_R$ in $P_R$ if and only if $tr_i$ is enabled at $M'$ in $C$.
Moreover, if $\state{ M_R'}{H'_R} {\ensuremath{\stackrel{t_i}{\rightsquigarrow}}}\state{ M_R}{H_R}$ and
$M'[tr_i\rangle M$ then $\state{ M_R}{H_R}$ corresponds to $M$.}
\end{theorem}
\begin{proof}
The proof is similar to the proof of Theorem~\ref{theoremMain}.
Due to the assumption that $t_i$ is enabled in the reverse direction in $R$ and $M'$ corresponds to the state
$\state{ M_R'}{H'_R}$, the content of places belonging to $P_R$ in $R$ and $C$ is equivalent.
Hence, we can focus on histories $H'_R$ and $H_R$ in $R$ and markings of transition and connection history
places in $C$.
As in the previous theorem, we consider two cases: \smallskip\\
-- Transition $t_i$ has been executed once, and $\#M'(h_i) = \#T_R - 1$.
After reversing it $H_R(t_i) = \emptyset$ in $R$, and
in $C$ we have to remove all the elements from its history place $h_i$, hence $M(h_i) = \emptyset$. \smallskip\\
-- Transition $t_i$ has been executed twice, $H'_R(t_i) = \{l_1, l_2\}$, $l_1 < l_2$,
and $\#M'(h_i) = 2 \cdot (\#T_R - 1)$. After reversing $t_i$, we remove the greater index for the
history obtaining $H_R(t_i) = \{l_1\}$. In $C$ we have two triples of the form $(k,j,i)$ in $M'(h_i)$
for every transition $t_j$. After reversing we have to remove the elements belonging to
$maxHist(h_i) = \{(k_m, j, i) \in M'(h_i) | k_m = max\{k|(k,j,i) \in M'(h_i)\}\}$ which
are related to the second execution. Clearly the remaining triples are related to the first execution
(i.e., the remaining $l_1$ in $H_R(t_i)$). Similarly, as in the previous proof, we can make use of formulas
defined in Section~\ref{sec:trans_bt} for computing the exact values. Moreover, we also need to
decrease the first coordinate in the triples in other transition history places, according to the Definition~\ref{def:co-eff}.
\end{proof}

\section{Cycles}
\label{sec.cycles}

In this Section we discuss possibilities of reversing of cycles in RPNs and corresponding CPNs.
We~now proceed to define cycles in reversing Petri nets.

\begin{definition}{\rm
A \emph{cycle} of reversible Petri net $(P, T, F, A, B)$ is a~sequence $x_0\ldots x_m$, with
$x_i\in{P\cup T}$ for $0\leq i\leq m$, such that
$F(x_i,x_{i+1})\neq\emptyset$ for $0\leq i<m - 1$, and $x_0=x_m$.
A cycle is \emph{simple} when no elements (except $x_0=x_m$) occurs more than once in it.
}\end{definition}

For the purpose of this paper, we assume that every cycles starts with a place (i.e., $x_0 \in P$).

\subsection{Infinite and finite cycles}

To distinguish cycles, which could be executed infinite and finite number of times, we define two types of transitions: those that transfer tokens and those that create bonds.

\begin{definition}{\rm
Let $(P, T, F, A, B)$ be a reversing Petri net and $t \in T$.
Transition $t$ is called:
\begin{itemize}
\itemsep=0.95pt
\item a \emph{transferring} transition if
$\bigcup_{p \in {}^\bullet t}{F(p,t)} = \bigcup_{p \in t^\bullet}{F(t, p)}$;
\item a \emph{bond-creating} transition if $t$ is not a transferring transition, i.e., there exists $p \in t^\bullet$ such that $\beta \in F(t, p)$ for some $\beta \in B$ and $\beta \not\in \bigcup_{p \in {}^\bullet t}{F(p,t)}$.
\end{itemize}
}\end{definition}

\begin{remark}
Note that, according to previous assumption (Definition~\ref{def:well-formed-RPN}),
the set $ t^\bullet$ consists of one element only.
\end{remark}

To create infinite cycles in RPNs only transferring transitions could be used. According to assumptions from Section~\ref{sec.RPN}, one token of each type can be preset in RPN, hence bond-creating
transitions can be fired only once. Even if this restriction would be relaxed, to obtain infinite execution of a~bond-creating transition,
initial marking of at least one of its input places would have to be infinite. This condition goes against the
definition of Petri nets in general. Hence,
a cycle, executed infinite number of times, can be created only by transferring transitions.

Unfortunately, infinite cycles in RPNs would cause problems with
infinite values of history, both in RPNs and CPNs corresponding to them. One of our goals was to eliminate infinite numbers from this model to avoid Turing power complexity, and - as a consequence - undecidability of decision problems. This is the first reason why this type of cycles is undesirable.

Moreover, and maybe even more importantly, when
we consider biological motivation, cycles created only by
transferring transitions are unnatural. No organism would waste energy on endless transportation of molecules. Substances are transported
only in order to finally perform some reactions or operations on them. Those reactions or operations are the goals of the transportation.

Because of the above reasons, we would focus on cycles created not only
by transferring transitions, but also at least one bond-creating transition.

\begin{definition}{\rm
Let $(P, T, F, A, B)$ be a reversing Petri net. It is called \emph{trans-acyclic} if it
does not contain any cycle consisting of transferring transitions only.
\label{def:trans-ac}
}\end{definition}

\begin{figure}[!b]
\begin{center}
\includegraphics[scale=0.65]{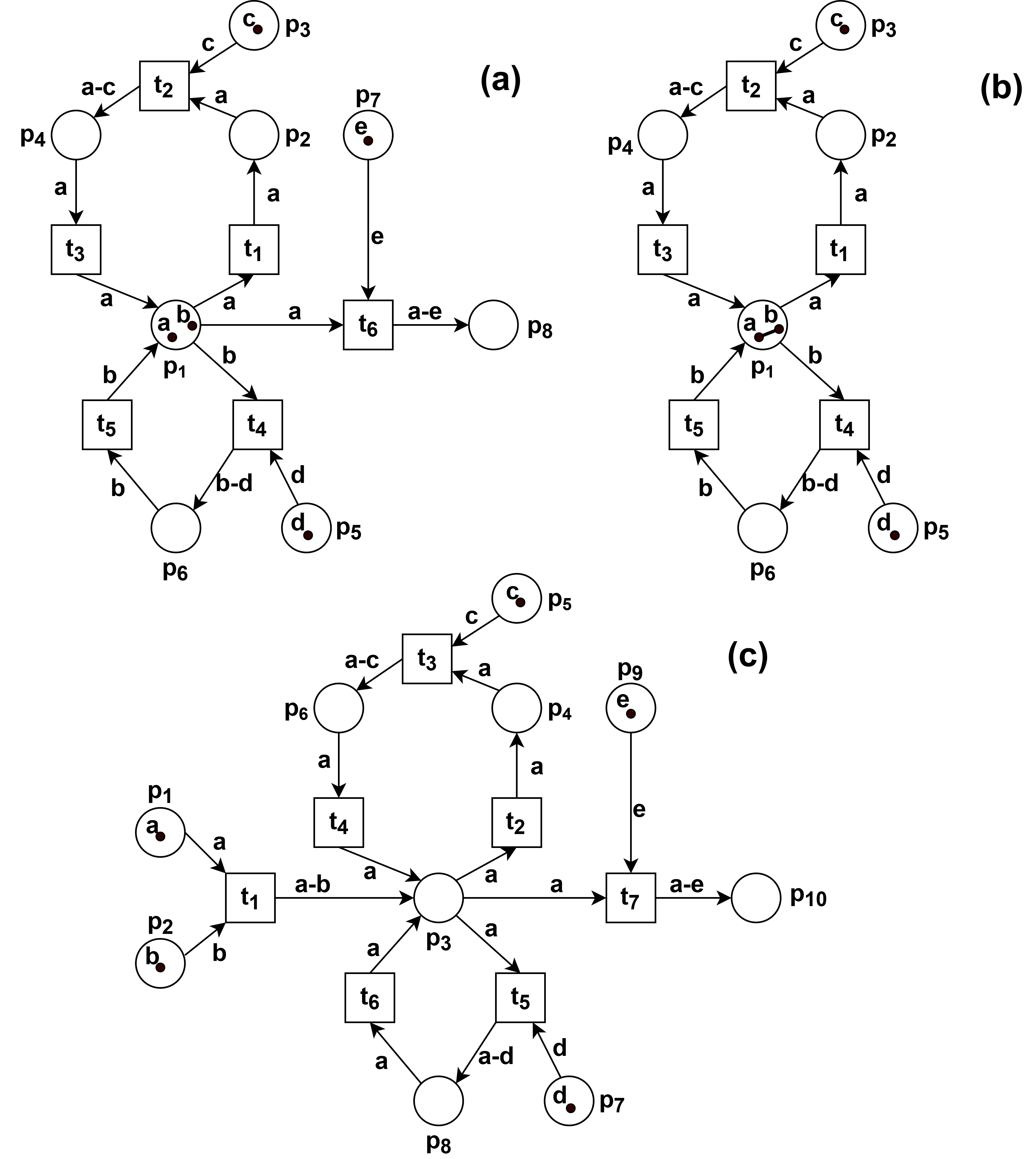}
\caption{Examples of trans-acyclic RPNs.}
\label{figure3}
\end{center}\vspace*{-3mm}
\end{figure}

\begin{remark}
Note that every cycle in trans-acyclic RPN has to contain at least one bond-creating transition.
Also, as assumed,
bonds cannot be recreated.
Consequently, every bond-creating transition, even in cycles, can be executed at most once.
Transferring transition can be executed at most two times -- it can be executed
twice only if it occurs in a cycle before a~bond-creating transition. Transferring transition following
a bond-creating transition in a cycle
can be executed only once.
Therefore, the state space of a trans-acyclic RPN is finite.
\label{remarkTwo}
\end{remark}

\begin{example}
For better understanding of the issue, look at Figure~\ref{figure3}a.
Transitions can fire in a sequence: $t_1t_2t_3$ - all of them
would be executed once. Then transition $t_1$ could be fired for the
second time (also transitions $t_4$ and $t_6$ are enabled). However,
after the second execution of $t_1$ no other transition is enabled
because place $p_3$ is empty and $t_2$ cannot fire.
\end{example}
Having in mind that every transition can be executed at most twice, notice that
for a given $RPN =(P,T,F,A,B)$ we have: if $k \in H(t)$ for some $t \in T$ then
$k \leq 2\cdot\#(T)$. For this reason any value of any history belongs to the set
$\Nat_b = \{0, 1, ..., 2\cdot\#(T)\}$.

From now on, only trans-acyclic RPNs are considered.

\subsection{Reversing of cycles in causal semantic}

Now, we consider how reversing of cycles is performed in various reversing semantics.

In \emph{out of causal} method (which has been only briefly mentioned in this paper)
every transition, which was executed, can be reversed, and cycles would not change that. From the point of view of cycles, this is not
very entrancing situation and that is the reason why \emph{out of causal} semantic is not considered in this paper.
In \emph{backtracking} only the recently fired transition could be reversed
and this is also not very intriguing when we consider trans-acyclic RPNs.

The most interesting case of cycles reversing in trans-acyclic RPNs is \emph{causal} reversing.
We can say that it lies between other two approaches.
Here, by
adopting different definition of dependence we can control,
to some point, the order in which transitions could be reversed.

According to the definition of structural dependence presented in Section~\ref{ssec:causal} (see
Definition~\ref{d:dep}) two transitions are  structurally dependent when at least one output place
of the first transition is also the input place of the second or vice versa.
However, the structural approach to dependencies is somewhat
strict. Please consider Figure~\ref{figure3}a. The dependence relation in this case is as follows: $Dep_{str} = \{(t_1, t_2), (t_1, t_3), (t_1, t_5), (t_2, t_3), (t_3, t_6), (t_3, t_4), (t_4, t_5), (t_5, t_6) \}$ (for clarity we do not specify the symmetrical elements).
After sequence of executions: $t_1t_2t_3t_4t_5$ only transition
$t_5$ can be reversed. Transition $t_3$ cannot be reversed because it
is dependent on $t_4$ (they have common place $p_1$), hence $t_4$ has to be
reversed first. However, when one consider changes of markings, it is
easy to notice that both cycles seem to be independent because
it is not important which bond ($a\bond c$ or $b\bond d$) is created first. Moreover, transitions $t_3$ and
$t_5$ even do not use the same tokens.

Presented example shows
that to distinguish dependent and independent transitions, instead of
using the structural dependence, more suitable approach
is to consider marking and tokens used by transitions. Hence, let us define
the marking-oriented dependence (this type of dependence
is investigated in details in~\cite{CyclesRPN}).

\begin{definition}{\rm
Consider \RPN $(P,T,F,A,B)$.
Transitions $t_1$ and  $t_2$ from $T$ are \emph{marking-oriented dependent} if there exist:
a place $p\in P$, a base $a\in A$ and
a reachable state $\langle M, H\rangle$ (such that $H(t_1)\neq \es$, and $H(t_2)\neq \es$) for which,
having $C=\connected(a,M(p))$, the following holds:
$C\cap \effects{t_1} \cap \effects{t_2} \neq \emptyset$.
}\end{definition}

According to the above definition, two transitions are marking-oriented dependent if they manipulate the same token, i.e.,
both components produced or transferred by those transitions contain the same base.
Notice, that it is not required that the token appears on the label of the arcs
in the net.
Since location of tokens is a dynamic aspect of Petri net, this type of dependence
is determined by the initial state of RPN and can be described only by observing the current marking.

Let us look again at Figure~\ref{figure3}a. According to marking-oriented definition of dependence transitions
$t_3$ and $t_5$ are independent because $t_3$ manipulates tokens
$a$ and $c$, when $t_5$ manipulates tokens $b$ and $d$.
Transitions $t_2$ and $t_6$ in Figure~\ref{figure3}c are depended
because they both transfer token $a$, however, this situation can be seen by looking at the net structure, it is not necessary to test the individual marking.
On the other hand, $t_3$ and $t_5$ at Figure~\ref{figure3}b are marking-oriented dependent because they both transfer components
containing token $a$. It is impossible to discover this dependence only by observing the structure of the net, without looking at its marking.

One can say that the marking-oriented dependence is finer.
However, implementation of this dependence requires large modifications
in the model, especially in a way of generation of CPNs from RPNs
and functions related to those CPNs. Even with structural dependence,
arcs and guards expressions are quite complex, with the marking-oriented
one they would be even more difficult to implement. Moreover,
marking-oriented dependence has other feature, which in some
situations may be considered as disadvantage. The sets of
dependent and independent transitions in one RPN may differ
between executions.
All of this together is the reason why we would like to find a different definition of dependence,
which is less strict than the structural one and would allow
some flexibility with reversing of cycles in trans-acyclic RPNs.
At the same time, to avoid large modifications of the RPN semantic
and the CPN generation, we need to obtain flexibility based on
the structure of RPN, not on the dynamics of the net.
It results in the co-dependence relation.

\begin{definition}{\rm
		Let $P_R=(P, T, F, A, B)$ be a reversing Petri net, and $t_1, t_2 \in T$. We say that $t_1, t_2$ are
		\emph{co-backward-conflicted} (or in \emph{co-backward-conflict relation}, or \emph{co-dependent}),
		denoted by $(t_1, t_2) \in Dep_{co}$ (and $(t_2, t_1) \in Dep_{co}$, as the relation is symmetric),
		if there exists a~place $p \in P$ where $p \in {t_1^\bullet} \cap {t_2^\bullet}$ and
		there exists a cycle in $P_R$ such that $p$ belongs to the cycle, and, moreover: at least one
		of the transitions $t_1$ and $t_2$ does not belong to any simple cycle.
		Additionally, we assume that $(t,t)\in Dep_{co}$ for every $t\in T$.
		
		We say that two transitions $t_1,t_2\in T$ are \emph{co-independent} when they are not in the
		co-backward-conflict relation, hence the co-independence is defined as follows: $Ind_{co}=T^2\setminus Dep_{co}$.
		
		Moreover, we define $Ind_{co}|_T = \{t\in T\,|\,\exists_{t'\in T} (t,t')\in Ind_{co}\}$
		as the set of
		all transitions for which there exist at least one independent transition in $T$.
		}\label{d:co-dep}
\end{definition}

With the assumption of co-dependence transitions $t_3$ and $t_5$ in
Figure~\ref{figure3}a are independent. Hence, after sequence of
executions: $t_1t_2t_3t_4t_5$ both $t_3$ and $t_5$ can be reversed in
any other. However, transitions $t_4$ and $t_6$ in
Figure~\ref{figure3}c are independent. They both are co-dependent
on $t_1$, and $t_1$ can be reversed only when both of them are undone earlier, but $t_4$ and $t_6$ can be rollbacked in any order. It
 cause some unwanted consequences explained further in this section.

\medskip
The huge advantage of co-dependence relation is the possibility of implementing it
quite easily in CPNs generated for RPNs - only minor changes are necessary.
First, similarly to structural dependence (Definition~\ref{d:dep}),
let us introduce a~set $P_{BHP} \subseteq P_{CHP}$, called a set of \emph{backward-conflicted history places}, as follows
$P_{BHP} =  \lbrace h_{ij} \mid t_i, t_j \in T_R; t_i, t_j \in Dep_{co}; i < j \rbrace$.
Then, we need to adjust the procedure described in Section~\ref{sec:rev}. Recall that in the case of structural dependence, while checking whether an action could be reversed, we examine the content of connection history places belonging to $P_{SHP}$,
in order to find the value $history_{ij}$. Now, for co-independent transitions we do not need to explore all those places.
Depending on the type of transition $t_i\in T$ we consider two possibilities:
\begin{itemize}
\item for $t_i\in Ind_{co}|_T$ we only need to explore places belonging to $P_{BHP}$
\item for $t_i\notin Ind_{co}|_T$ we take into account the following set of connection history places
$P_{CHP}\setminus \lbrace h_{ij} \mid t_i, t_j \in T_R; t_i, t_j \in Dep, t_j\in Ind_{co}|_T\rbrace$.
\end{itemize}

\begin{figure}[ht]
\vspace*{-2mm}
\begin{center}
\includegraphics[scale=0.69]{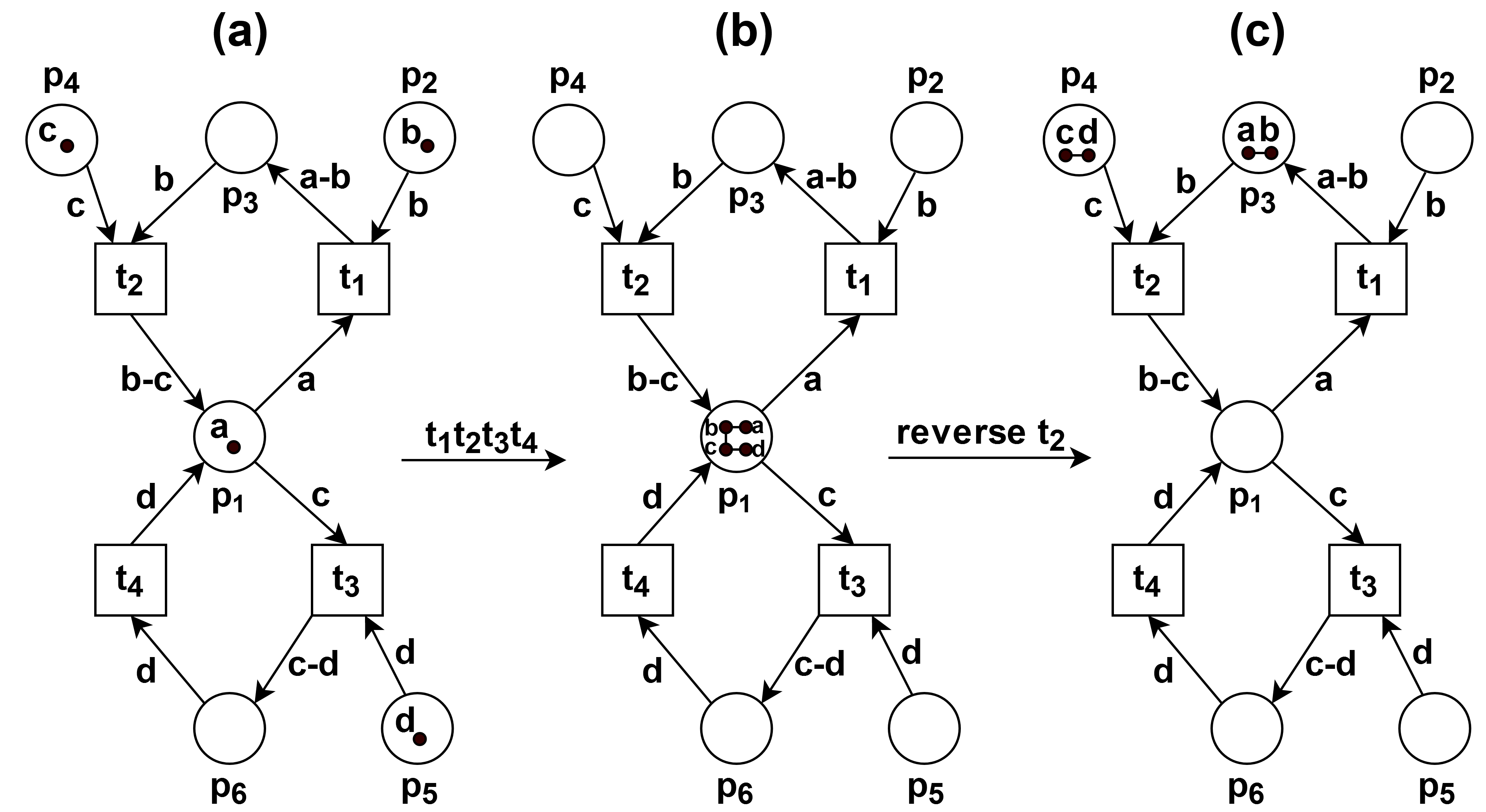}\vspace{-1mm}
\caption{A net which cannot reach (by reversing) its initial marking after reversing co-independent transitions $t_2$ and $t_4$.
The connected component (molecule) in place $p_1$ in part (b)
is $a\bond b\bond c \bond d$.}
\label{figure4}
\end{center}\vspace*{-3mm}
\end{figure}

Unfortunately, the example presented in Figure~\ref{figure4}
shows that in the co-backward-conflict relation approach
a system cannot always be brought back to the initial state.

Let us look at the example depicted in Figure~\ref{figure4}.
Part (a) shows
the initial marking, while part (b) the marking after the execution of transitions sequence
$t_1t_2t_3 t_4$ (in place $p_1$ we have connected component:
$a\bond b\bond c \bond d$). Note that transitions $t_2$ and $t_4$ are co-independent,
hence they can be reversed in any order. Let us reverse transition $t_2$ as the first one (see
Figure~\ref{figure4}c).
In the marking depicted in (c) base $d$ is still bonded with base $c$
in place $p_4$. Please notice that we cannot move
base $d$ from place $p_4$ by reversing -- we can say that base $d$ is \textit{stuck} in place $p_4$. At the same time, presence of $d$ in place
$p_1$ is
required to reverse transition $t_4$, and only then $t_3$ could be
rollbacked.
Therefore we cannot reach the initial marking only by reversing of transitions.
To obtain the initial marking in the presented situation (Figure~\ref{figure4}c), we should execute transition $t_2$ in forward direction  --
it would mean that we actually ``undo the
reversing'', and then reverse $t_4$ first.

This example shows that the order of reversing co-independent
transitions is crucial, which is against the idea of causal reversing
and, unfortunately, we need to be careful with this definition of dependence.

\section{Conclusions and future work}
\label{sec.cons}

This paper is an improved version of \cite{BGMPPP}, enriched with
discussion related to cycles.
Here, in comparison to \cite{BGMPPP}, we focus more on backtracking and causal reversing semantics, because they are more interesting in the
context of cycles. Moreover, we change the form of history in RPNs,
from the single integer to a set of numbers. Furthermore, formal
proofs of generation of CPNs from RPNs are presented in this paper.

In the second part of the paper we discuss the possibility of introduction of
cycles to RPNs, and thus their introduction to CPNs generated from
RPNs. It turned out that the most interesting case is reversing of
cycles in the causal semantic, where possibility of reversing depends on
definition of dependence. Three definitions of dependence have been
presented: structural, marking-oriented and co-backward conflict.
The structural one is most \textit{strict}, reversing
of cycles is the same as in backtracking. With the marking-oriented
one, in some cases cycles can be reversed in different order than
they were executed in
forward direction. Unfortunately, this dependence is based on dynamic behaviour
and its introduction to the current version of CPNs generation
algorithm is not possible without large modifications. We tried
to find a new type of dependence, which would allow more flexibility
in cycles reversing but would be based on the structure of RPNs.
It resulted in co-backward conflict dependence. Unfortunately,
we discovered that this dependence lead to unwanted behaviour,
and should be used with caution. We would like to find a different,
structure based, definition of dependence, which would allow
"proper" causal reversing of cycles. It rises a question, if is it even
possible? We would study it more in the future.

In this paper, we applied the limitation of the number of bases of a~given type to one element.
However, we believe that the presented results would be valid even if this limitation is lifted (multitoken semantics).
This would rise a need of token identification, but the overall behaviour of the net would remain unchanged.
Moreover, the extension of the formulas for the enumeration of indexes in histories, analogous to the present ones, would be needed.

As a general aim, we plan on implementing an algorithmic translation that transforms
RPNs to CPNs in an automated manner using the transformation techniques discussed in this
paper.
We also aim to explore how our framework applies in fields outside computer science, since the
expressive power and visual nature offered by Petri nets coupled with reversible computation has
the potential of providing an attractive setting for analysing systems (for instance in biology,
chemistry or hardware engineering).

\subsection*{Acknowledgements}
We are immensely grateful to \L ukasz Mikulski
for his valuable insights and suggestions during discussions of this work.
Furthermore, we would like to thank Anna Philippou and Kyriaki Psara for their ideas and working together on the preliminary, unpublished version of the paper.

\end{document}